%% file: main.tex
\documentclass[letterpaper, 10pt, conference]{IEEEtran}

\IEEEoverridecommandlockouts 

\input{preamble}

\begin{document}

\title{Observational Learning with a Budget}

\author{
  \IEEEauthorblockN{Shuo Wu}
  \IEEEauthorblockA{ECE Department\\
                    University of Illinois Chicago\\
                    Email: swu99@uic.edu}
  
  \and
  \IEEEauthorblockN{Pawan Poojary and Randall Berry}
  \IEEEauthorblockA{ECE Department\\
  Northwestern University, Evanston IL\\
Email: \{pawanpoojary2018@u, rberry@\}.northwestern.edu}
}

\maketitle

\begin{abstract}
We consider a model of Bayesian observational learning in which a sequence of agents receives a private signal about an underlying binary state of the world. Each agent makes a decision based on its own signal and its observations of previous agents.  A central planner seeks to improve the accuracy of these signals by allocating a limited budget to enhance signal quality across agents. We formulate and analyze the budget allocation problem and propose two optimal allocation strategies. At least one of these strategies is shown to maximize the probability of achieving a correct information cascade.
\end{abstract}

\input{sections/01_introduction}
\input{sections/04_Preliminaries}
\input{sections/06_Markovian_analysis}
\input{sections/07_Single_sideSignal_Quality_Improvement}

\input{sections/08_Double_side_quality_improvement}

\bibliographystyle{IEEEtran}
\bibliography{IEEEfull,bibfile.bib}

\clearpage

\input{sections/11_Appendix}

\end{document}

%% file: preamble.tex
\usepackage[utf8]{inputenc} 
\usepackage[T1]{fontenc}
\usepackage{url}
\usepackage{ifthen}
\usepackage{cite}
\usepackage[cmex10]{amsmath} 
\usepackage{cleveref}

\usepackage{graphicx} 
\usepackage{subcaption} 
\usepackage[english]{babel}
\usepackage{amsmath,amssymb}
\usepackage{amsmath,amsfonts,amssymb,amsthm,bm,cleveref}
\usepackage{xfrac}
\usepackage[font={it,rm}]{caption}
\usepackage{tikz}
\usetikzlibrary{arrows.meta}
\usetikzlibrary{arrows,decorations.markings}
\usetikzlibrary{decorations.pathreplacing,angles,quotes,shapes}
\usepackage{mathtools}

\DeclarePairedDelimiter\abs{\lvert}{\rvert}

\newtheorem{assumption}{Assumption}
\newtheorem{lemma}{Lemma}
  \newtheorem{theorem}{Theorem}
  \newtheorem{prop}{Proposition}

\newtheorem{definition}{Definition}
\newtheorem{property}{Property}  

\newcommand{\prob}{\mathbb{P}}
\newcommand{\floor}[1]{\lfloor #1 \rfloor}

\newcommand{\summ}[2]{\sum_{#1}^{#2}}
\newcommand{\nat}{\mathcal{N}^\star}
\newcommand{\texit}[1]{\textit{#1}}

%% file: sections/01_introduction.tex
\section{Introduction}
\label{sec:intro}


Consider that an item, which could either be of a ``good'' or a ``bad'' quality, is up for sale in a market where agents arrive sequentially and decide whether to buy the item, with their choice serving as a recommendation for later agents. While the quality of the item is unknown to the agents, every agent has its own prior knowledge of the item’s quality in the form of its private belief. Each agent then makes a pay-off optimal decision based on its own prior knowledge and by observing the choices of its predecessors. Such models of ``observational learning'' were first studied by \cite{bhw,banerjee,welch} under a Bayesian learning framework wherein each agent's prior knowledge is in the form of a privately observed signal about the pay-off-relevant state of the world, which in this case is the item's quality, and is generated from a commonly known probability distribution. A salient feature of such models is the emergence of \emph{information cascades} or \emph{herding}, i.e., at some point, it is optimal for an agent to ignore its own private signal and follow the actions of the past agents. Subsequent agents then follow suit due to their homogeneity. As a result, from the onset of a cascade, the agents' actions do not reveal any information conveyed to them by their private signals; hence learning stops. In this way, the phenomenon of cascading could impede the agents from learning the socially optimal (correct) choice, i.e., the agents could herd to a wrong cascade.


The basic Bayesian learning models in \cite{bhw,banerjee,welch} considered a binary state-space for the underlying state of the world as well as a binary signal-space for agents' private signals, where an agent's
signal can be viewed as the result of passing the true state over a binary channel.  In this paper, we consider a similar model, where additionally, we consider that a fixed \emph{budget} is provided for improving agents' private signal qualities, i.e. for making the channel ``less noisy.''  For example, this could model a setting where limited resources such as time and money could be spent researching the underlying quality of the item and sharing the results with the agents, resulting in them having a better chance of receiving the correct private signal. Our problem is that of ``information design'', where an entity  must allocate the budget towards improving each of the two signal qualities in a manner that maximizes the probability of a correct cascade.  
We characterize the optimal use of this budget. 

In related work, \cite{itai2023} considers a similar entity that designs the information structure for the agents’ private signals, but instead is self-interested and aims to maximize its own utility, which may not coincide with maximizing the agents' chances of optimal learning. Moreover, unlike in \cite{itai2023}, our entity is given an information structure to begin with, which reflects signal qualities in the current market and is only allowed to improve them, subject to budget constraints. The work in \cite{smithNsorensen} extended the basic models in \cite{bhw,banerjee,welch} by allowing for a continuous signal-space  with signals having unbounded likelihood ratios, and showed that this could result in learning to occur with probability one. Our model maintains the assumptions of \cite{bhw,banerjee,welch}, i.e., discrete bounded private signals, which always leads to a positive probability that learning fails. Further, as in \cite{bhw,banerjee,welch}, we consider a setting in which all agents can perfectly observe all prior actions. Whereas, other variations of this basic model include assuming agents do not observe all previous agents’ actions \cite{acemouglu,song}, allowing for imperfect observations \cite{Tho,poojary2020,pawanWiopt2023}, and others \cite{pawan_congestion_2022,VijayShih2018,Achilleas2022}. 

Our approach requires characterizing the probability that agents enter into a correct cascade as a function of the private signal qualities. As a part of this analysis, we also show an interesting property that the function behaves differently based on the rationality of an underlying constant. Due to this peculiar property, assigning all available budget for improving the signal qualities may be a strictly suboptimal strategy, somewhat counterintuitively.

%% file: sections/04_Preliminaries.tex
\section{Preliminaries}

\subsection{Model}

We consider a model similar to \cite{bhw} in which there is a countable sequence of agents, indexed $i = 1,2, \ldots$ where the index represents both the time and the order of actions. Each agent $i$ takes a publicly observable action $A_i$ of either buying $(Y)$ or not buying $(N)$ a new item that has an unknown true value $(V)$, which is equally likely to be either good $(G)$ or bad $(B)$. Each agent $i$ receives a pay-off of $v>0$ if $A_i=Y$ and $V=G$ and a pay-off of $-v$ if $A_i=Y$ and $V=B$. If $A_i=N,$ its pay-off is zero.

To incorporate agents' private beliefs about the new item, every agent $i$ receives a private signal $S_i \in \{ H \, \text{(high)}, L \, \text{(low)}\}$. This signal, as shown in Figure \ref{BC}, partially reveals the information about the true value of the item through a binary channel with $p_1 \triangleq \prob (H \vert G)$ and $p_2 \triangleq \prob (L  \vert B)$, respectively, denoting the probabilities of receiving the correct signal when $V= G$ and when $V= B$. We refer to the ordered pair $(p_1,p_2)$ as the signal \emph{qualities}. Here, we assume $\sfrac{1}{2} < p_1, p_2 < 1$, which implies that the signal is informative but not revealing. Moreover, the sequence of private signals $\{S_1,S_2,\ldots\}$ is assumed to be \emph{i.i.d.} given the true value $V$. Each agent $i$ seeks to maximize its expected pay-off given its private signal $S_i$ and the history of past observations $ \mathcal{H}_{i-1} := \{A_1,A_2, \ldots, A_{i-1} \}$.


\begin{figure}[ht!]
\centering
\begin{tikzpicture}[scale=1.4]
\draw [decoration={markings,mark=at position 1 with {\arrow[scale=2,>=stealth]{>}}},postaction={decorate}] (0,0) -- (1.5,0);
\draw [decoration={markings,mark=at position 1 with {\arrow[scale=2,>=stealth]{>}}},postaction={decorate}] (0,-1) -- (1.5,-1);
\draw [decoration={markings,mark=at position 1 with {\arrow[scale=2,>=stealth]{>}}},postaction={decorate}] (0,0) -- (1.5,-1);
\draw [decoration={markings,mark=at position 1 with {\arrow[scale=2,>=stealth]{>}}},postaction={decorate}] (0,-1) -- (1.5,0);
\node at (1.65,0) {$H$};
\node at (1.65,-1) {$L$};
\node at (-0.15,0) {$G$};
\node at (-0.15,-1) {$B$};
\node at (-0.35,-0.5) {$V$};
\node at (1.85,-0.5) {$S_i$};
\node at (0.75,0.15) {\scalebox{0.85}{$p_1$}};
\node at (0.75,-1.15) {\scalebox{0.85}{$p_2$}};
\node at (0.58,-0.24) [rotate= -33.7] {\scalebox{0.83}{$1-p_1$}};
\node at (0.58,-0.78) [rotate= 33.7] {\scalebox{0.83}{$1-p_2$}};
\end{tikzpicture}
\caption{\small The channel through which agents receive private signals.}
\label{BC}
\end{figure}
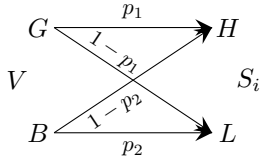

By allowing for arbitrary values of $p_1,p_2 \in (0.5,1)$, our model generalizes similar models studied in past works \cite{bhw,banerjee,Tho,poojary2020,pawan_congestion_2022,pawanWiopt2023,le2018,peres2017,Wu2015,Achilleas2022}, which assume $p_1 = p_2$ and are thereby restricted to symmetric information channels.

\subsection{Optimal decision and cascades} 
For the $n^{\text{th}}$ agent, the history of past actions $\mathcal{H}_{n-1} $ and its private signal $S_n$ form its information set $\mathbb{I}_n := \{S_n, \mathcal{H}_{n-1} \}$. As the first agent does not have any observation history, he always follows his private signal, i.e., he buys the item if and only if the signal is $H$. For agent $n \geq 2$, the Bayes' optimal action, $A_n$ is chosen according to the hypothesis ($V=G$ or $B$) that has the higher posterior probability given the agent's information set $\mathbb{I}_n$. Let $\gamma_n (S_n, \mathcal{H}_{n-1} ) \triangleq \mathbb{P} (V=G \vert S_n, \mathcal{H}_{n-1} )$ denote the posterior probability for the item being good. Then, the Bayes' optimal decision rule is
\begin{align}
A_n = \begin{cases} 
Y, \quad & \text{if} \; \; \gamma_n  > 1 / 2, \\
N, \quad & \text{if} \; \; \gamma_n  < 1 / 2, \\
\text{follows}\;\; S_n,  \quad & \text{if} \; \; \gamma_n = 1 / 2.
\end{cases} 
\label{bayes_decison}
\end{align}
Note that when $\gamma_n = 1/2$, an agent is indifferent between the actions. Similar to prior works \cite{banerjee,Tho,poojary2020,pawan_congestion_2022,le2018}, our decision rule in this case follows the private signal $S_n$, unlike \cite{bhw} where they employ random tie-breaking.

\begin{definition}
An information cascade is said to occur when an agent's decision becomes independent of its private signal.  
\end{definition}

It follows from \eqref{bayes_decison} that, agent $n$ cascades to a $Y$ $(N)$ if and only if $\gamma_n  > 1 / 2 $ $( < 1/2)$ for all $S_n \in \{ H,L \}$. The other case being $\gamma_n  > 1 / 2$ for $S_n=H$ and less than $1/2$ for $S_n=L$; in which case, agent $n$ follows $S_n$. A more intuitive way to express this condition is by using the information contained in the history $\mathcal{H}_{n-1}$ observed by agent $n$ in the form of the public likelihood ratio $l_{n-1} (\mathcal{H}_{n-1}) \triangleq \mathbb{P}(\mathcal{H}_{n-1} \vert B) / \mathbb{P}(\mathcal{H}_{n-1} \vert G) $ as follows.

\begin{lemma}
Agent $n$ cascades to a \textit{Yes} \, (\textit{No}) action if and only if $l_{n-1} < \frac{1-p_1}{p_2} $$\big( l_{n-1} \!>\!\! \frac{p_1}{1-p_2} \big)$\! and otherwise follows its private signal $S_n$.
\label{lemma1}
\end{lemma}

This follows by expressing $\gamma_n$ in terms of $l_{n-1}$ using Bayes' law, and then using the condition on $\gamma_n$ for a $Y \, (N)$ cascade. If agent $n$ cascades, then its action $A_n $ does not provide any additional information about the true value $V$ to the successors over what is contained in $\mathcal{H}_{n-1}$. As a result, $l_{n+i} = l_{n-1} $ for all $i=0,1,2,\ldots$ and hence they remain in the cascade, which leads us to the following property, also exhibited by prior models \cite{bhw,banerjee,welch,Tho,poojary2020,pawanWiopt2023}.
\begin{property}
Once a cascade occurs, it lasts forever.
\label{prop1}
\end{property}

On the other hand, if agent $n$ does not cascade, then Property \ref{prop1} and Lemma \ref{lemma1} imply that all the agents $i\leq n$ follow their own private signals ignoring the observations of their predecessors. Mutual independence between the private signals results in the likelihood ratio $l_n$ depending only on the number of $Y$'s ($n_Y$) and $N$'s ($n_N$) in the observation history $\mathcal{H}_{n}$. Specifically, $l_n = \big( \frac{p_2}{1-p_1} \big)^{h_n}$ where 
\begin{align}
h_n =& \; n_N - a n_Y, \label{h_n}  \\ 
a :=& \; \log \Big( {\small \text{$\frac{p_1}{1-p_2}$}} \Big) / \log \Big( { \small \text{$ \frac{p_2}{1-p_1}$}} \Big). \notag
\end{align}
We refer to $a$ as the \emph{cascade constant}. We summarize this in the following property.
\begin{property}
Until a cascade occurs, each agent follows its private signal and $h_n$ defined in \eqref{h_n} is a sufficient statistic of the information contained in the past observations.
\end{property}
Note that if $p_1 = p_2$ (symmetric channel), then $a = 1$ and $h_n$ in \eqref{h_n} is the {\it unweighted} difference, $n_N- n_Y,$ which is the case for the models in \cite{bhw,banerjee,Tho}. Interestingly, if the signal quality for $V=G$ is less than for $V=B$, i.e., if $p_1 < p_2$, then $a > 1$,  which implies that the information conveyed by a $Y$ in the observation history is greater by a factor of $a$ compared to that conveyed by a $N$. Whereas, $a < 1$ if $p_1 > p_2$.

From Lemma \ref{lemma1}, it follows that for all times $n$ until a cascade occurs, $-1 \leq h_n \leq a$, and the update rule for $h_n$ is given by
\begin{align}
h_n = \begin{cases} h_{n-1} + 1 \quad & \text{if} \;\; A_n = N,\\
h_{n-1}-a  \quad & \text{if} \;\; A_n = Y.
\end{cases}
\label{mc_update}
\end{align}
Once $h_n >a $ ($ < -1 $), a $Y \,(N)$ cascade begins and $h_n$ stops updating (Property \ref{prop1}). A cascade is \emph{correct} if the cascade action is optimal for the realized value of $V \in \{G,B\}$. Otherwise, it is \emph{wrong}. 
Now, given $V,$ let the probability that a $Y$ $(N)$ cascade begins be denoted by $\mathbb{P}(Y_{\text{cas}} \, \vert \, V)$ $\big( \mathbb{P}(N_{\text{cas}} \, \vert \, V) \big)$. Here, $ \mathbb{P}(Y_{\text{cas}} \, \vert \, V) = 1- \mathbb{P}(N_{\text{cas}} \, \vert \, V)$ as it can be shown that the process $\{h_n\}$ described in \eqref{mc_update} exits the range $[-1,a]$ w.p. $1$. Given that both $V$'s are equiprobable, the unconditional probability of a correct cascade, denoted by $\mathbb{P}_{\text{cc}}(p_1, \, p_2)$  is given by
\begin{align}
\label{eq:cc}
\mathbb{P}_{\text{cc}}(p_1, \, p_2) =  \textstyle  \frac{1}{2}\mathbb{P}(Y_{\text{cas}} \, \vert \, G) + \frac{1}{2}\mathbb{P}(N_{\text{cas}} \,\vert\, B).
\end{align}
 
\subsection{Optimal Budget Allocation Problem}

Consider the problem where a fixed budget $b$ is provided for improving the signal qualities, where the objective is to maximize the probability of a correct cascade. Denote the new signal qualities by $p_1' := p_1 + c_1$ and $ p_2' := p_2 + c_2$, where the respective quality improvements, $c_1,c_2 \geq 0$, must satisfy the budget constraint: $c_1 + c_2 \leq b$. Additionally, the signal qualities may not be improved to the point where either of them become fully revealing, i.e., $p_1', p_2' \neq 1$, as in such cases learning occurs with certainty. Then, the optimization problem can be stated as follows.
\begin{equation}
\begin{aligned}
\text{Find} \;\;\, (c_1^*,c_2^*) := \underset{(c_1,\,c_2)}{\arg \max} & \;\, \mathbb{P}_{\text{cc}}(\,p_1+c_1, \, p_2+c_2) \\
\text{s.t.} & \begin{cases} 
\, c_1+c_2 \leq b, \\
\, 0 \leq c_1 < (1-p_1), \\
\, 0 \leq c_2 < (1-p_2).
\end{cases}
\end{aligned} 
\label{eq:optimization}
\end{equation}
There is an assumption we want to stress as follows.
\begin{assumption}
    We assume $p_1 \leq p_2$ so that $a\geq 1$. Unless otherwise specified, we assume $p_1 < p_2$ in the theoretical results and discuss $p_1=p_2$ as a special case.
\end{assumption}
The assumption is without loss of generality because $p_1,p_2$ are indistinguishable in terms of their impacts on $\mathbb{P}_{\text{cc}}(\,p_1, \, p_2)$.

%% file: sections/06_Markovian_analysis.tex
\section{Markovian Analysis of Cascade Probability}

Before addressing the optimization in (\ref{eq:optimization}), in this section we use a Markov chain formulation as in \cite{le2018} to analyze the probability of a correct cascade for an arbitrary choice of $p_1\leq p_2$. In the following section we will then use this to study the optimization of $p_1$ and $p_2$.

\subsection{Markovian Reformulation}



Conditioned on a given value of $V$, $\{h_n\}$ is a Markov chain. From the previous section, it follows that until a cascade  occurs, $h_n$ increases by 1 if $A_n= N$ and decreases by $a$ if $A_n=Y$, where the probability of these events corresponds to the probability of observing $L$ and $H$, respectively. If $h_n > a$ or $h_n < -1$, then a cascade occurs and $h_n$ stops evolving, corresponding to a $N$ or $Y$ cascade, respectively. Figure~\ref{random_walk} illustrates this Markov chain when $V=G$. 
 



\begin{figure}[h!]
\vspace{-0.5mm}
\centering
\begin{tikzpicture}[scale=1.65]
\def\a{1.5}

\draw [decoration={markings,mark=at position 1 with {\arrow[scale=2,>=stealth]{>}}},postaction={decorate}] (0,0) -- (2.45,0);
\draw [decoration={markings,mark=at position 1 with {\arrow[scale=2,>=stealth]{>}}},postaction={decorate}] (0,0) -- (-2.45,0);

\draw (0,0.05) -- (0,-0.05);
\draw (-\a,0.05) -- (-\a,-0.05);
\draw (\a,0.05) -- (\a,-0.05); 
\draw (-1,0.05) -- (-1,-0.05); 

\node at (0+0.1,-0.17) {$0$};
\node at (1,-0.15) {$1$};
\node at (2,-0.15) {$2$};
\node at (-1+0.15,-0.15) {$-1$};  
\node at (\a+0.1,-0.15) {$a$};     
\node at (-\a-0.1,-0.15) {$-a$}; 
\draw [->] (0,0) to [out=80,in=100] (1,0); 
\node at (0.5,0.4) {\scalebox{0.95}{$1-p_1$}}; 

\draw [->] (1,0) to [out=80,in=100] (2,0); 
\node at (1.5,0.4) {\scalebox{0.95}{$1-p_1$}}; 

\draw [->] (0,0) to [out=-110,in=-70] (-\a,0); 
\node at (-0.75,-0.56) {\scalebox{0.95}{$p_1$}};

\draw[dashed] (-1,-0.5) -- (-1,0.5); 
\draw[dashed] (\a,-0.5) -- (\a,0.5); 

\end{tikzpicture}
\setlength{\belowcaptionskip}{0pt}
\caption{\small Example transition diagram for $\{h_n\}$ when $V=G$.}  
\label{random_walk}
\end{figure}
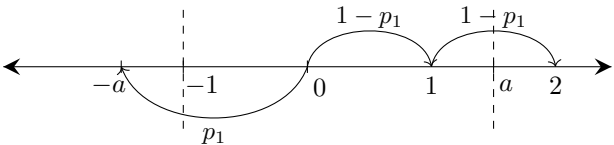

The following lemma show a specific structure of the sample paths for $\{h_n\}$.

\begin{lemma}\label{lemma: action region}
    For any $n$, suppose that the observation history $\mathcal{H}_{n-1}$ does not result in a cascade. Then, for any $m<n$, $A_m = Y$ only if $h_m\in [a-1,a]$. Similarly, $A_m = N$ only if $h_m\in [-1,a-1]$.
\end{lemma}

This shows that the non-cascade region of the state space for $\{h_n\}$ can be divided into two intervals, that only overlap at the single point $a -1$. Except at this point there is only one action available that avoids entering into a cascade.  Further, we note that if $a$ is irrational, then $h_n$ will never equal $a-1$, while if $a$ is rational, this can occur. Hence, the calculation of the cascade probabilities differs in these two cases, which we look at next. 

\subsection{Irrational Cascade Constant}

If $a$ is irrational, then $h_n$ will never equal $a-1$ and so if there is only one possible observation history that does not end in a cascade at any time $n$. Furthermore, suppose an observation history results in a Y-cascade (crossing the left boundary) with a fixed number of \texit{No} actions, then the number of \texit{Yes} actions in the observation history is unique. 


Such a property facilitates the calculation of the correct cascade probability. We denote the probability of a $Y$-cascade occurring with exactly $i$ \textit{No} actions in the observations as $\prob^N_i(Y_{\mathrm{cas}} \,\vert\, G)$ when $V=G$. Similarly, when $V=B$, we denote this probability as $\prob^N_i(Y_{\mathrm{cas}} \,\vert\, B)$. By definition, the probabilities of a Y-cascade given the true state $V$ are expressed as:
\[
\prob(Y_{\mathrm{cas}} \mid G) = \sum_{i=0}^{\infty} \prob^N_i(Y_{\mathrm{cas}} \mid G), 
\]
\[
\prob(Y_{\mathrm{cas}} \mid B) = \sum_{i=0}^{\infty} \prob^N_i(Y_{\mathrm{cas}} \mid B).
\]
We now present the formulas for calculating these conditional probabilities. Note that given these probabilities, the correct cascade probability then follows from (\ref{eq:cc}) noting that $ \prob(N_{\mathrm{cas}} \mid B) = 1 -  \prob(Y_{\mathrm{cas}} \mid B)$.

\begin{prop} \label{coro: prob of CC}
Given $p_1, p_2$, when $a$ is an irrational number, 
\begin{align*}
\prob(Y_\mathrm{cas} \mid G) &= \sum_{i=0}^{\infty} (1 - p_1)^i p_1^{k_i} \;\; \text{and} \\
\prob(Y_{\mathrm{cas}} \mid B) &= \sum_{i=0}^{\infty} p_2^i (1 - p_2)^{k_i},
\end{align*}
where $k_i = \floor{\frac{i+1}{a}}+1$.
\end{prop}

In the following sections, we assume $k_i = \floor{\frac{i+1}{a}}+1$ and do not repeat this.

\subsection{Rational Cascade Constant}

When $a$ is rational, given a fixed number of \textit{No} actions, multiple sample paths may lead to a $Y$-cascade, as $h_n$ can move either left or right at the point $a - 1$. This differentiates the rational case from the irrational one. Consequently, we analyze the problem in the rational case iteratively.

The heuristic idea is as follows. If $h_n$ never reaches $a-1$, the possible observation history remains unique at each time $m\leq n$, and so the cascade probability can be found as in the irrational case. Suppose that $h_n$  starts at $a - 1$ and consider three possible outcomes: crossing the left boundary, crossing the right boundary, or returning to the starting point $a - 1$. If $h_n$ returns to the starting point, the probability of crossing the left or right boundary remains unchanged from the initial state. Otherwise, this probability contributes to the probability of either a $Y$-cascade or a $N$-cascade.

The following proposition presents the resulting formulas.

\begin{prop}\label{prop: irrational prob}
    Suppose that $a$ is rational and satisfies $a=r/q$, for integer values of $r$ and $q$ with no common factors, then 
    \begin{align}
        \prob(Y_{\mathrm{cas}} \mid G) =  \frac{\summ{i=0}{r-1}(1-p_1)^i p_1^{k_i}}{1 - 2p_1^{q}(1-p_1)^{r}}\label{eq: Y-cascade probability rational}, \\
        \prob(Y_{\mathrm{cas}} \mid B) = \frac{\sum_{i=0}^{r-1}{ p_2^i (1-p_2)^{k_i}}}{1 - 2(1-p_2)^q p_2^{r}}. \label{eq: N-cascade probability rational}
    \end{align}
\end{prop}

Note that the numerators in (\ref{eq: Y-cascade probability rational}) and (\ref{eq: N-cascade probability rational}) correspond to the sum of the first $r-1$ terms in a geometric sequence. \footnote{Since $a>1$, it must be that $r>1$, and hence these contain at least one term.}  In the case of $V=G$, these geometric terms are  $(1 - p_1)^i p_1^{k_i}$ for $i = 1, \ldots, r - 1$, which correspond to the probability of $i$ \textit{No} actions and $k_i$ \textit{Yes} actions.  The term $2 p_1^q (1 - p_1)^r$ in the denominator is twice the probability of returning to the point $a - 1$. This aligns with the previous heuristic argument: starting from $a - 1$, $h_n$ either returns to $a - 1$ or crosses the boundary.

Note that as $r \rightarrow \infty$, equations~(\ref{eq: Y-cascade probability rational}) and (\ref{eq: N-cascade probability rational}) converge to their corresponding probabilities in the irrational case. A natural question arises: can we simplify the problem by analyzing only the formula for the irrational case? The answer is yes, provided that small errors are acceptable and that finite values of the cascade constant are excluded.

\begin{prop}\label{prop: irrational case approaches rational case}
    We use the subscript "irra" to denote the probability computed using the formula for an irrational $a$ and "ra" to denote the probability computed using the formula for a rational value of $a$. For any \( \epsilon > 0 \),  except for at most  $N \leq \log_2((1/\epsilon)^2)$ values of \( a \), 
    \begin{align*}
        \abs*{\prob_{\text{ra}}(Y_{\mathrm{cas}} \mid G) - \prob_{\text{irra}}(Y_{\mathrm{cas}} \mid G)} < \epsilon,\\
        \abs*{\prob_{\text{ra}}(Y_{\mathrm{cas}} \mid B) - \prob_{\text{irra}}(Y_{\mathrm{cas}} \mid B)} < \epsilon.
    \end{align*}
\end{prop}

Suppose that $\epsilon = 10^{-6}$. Then, there are fewer than 400 rational numbers for which the $Y$-cascade probability differs by more than $\epsilon$ between the rational and irrational cascade constant cases, and the same holds for the $N$-cascade probability. Given the complexity of the cascade constant formula, 
\[
a = \frac{\log\left(\frac{p_1}{1 - p_2}\right)}{\log\left(\frac{p_2}{1 - p_1}\right)},
\]
it is highly unlikely that the values of $p_1$ and $p_2$ would result in a cascade constant $a$ that coincides with one of these 400 rational numbers.

This observation is further supported by our numerical experiments, one of which is illustrated in Figure~\ref{fig:comparison}. Except for the special case where $p_1 = p_2$, there is no significant difference between the formulas for the irrational and rational cascade constants when $p_1 = 0.7$ and $p_2$ iterates over the interval from $0.5$ to $0.1$ with a step-size of $0.001$. At the point $p_1 = p_2 =0.7$, the formula of rational cascade constant case gives the probability of the correct cascade as $0.84$, and the irrational cascade formula gives $0.75$

Based on Proposition~\ref{prop: irrational case approaches rational case}, we ignore the influence of the rationality of the cascade constant in the following analysis except when $p_1=p_2$ and otherwise consider only the formulas for cascade probability with an irrational cascade constant. This is important as otherwise, we would have to account for the fact that the probability of the correct cascade has discontinuities at every rational number, which would complicate optimizing this quantity. 

\begin{figure}[h!]
    \centering
    \begin{subfigure}[b]{1\linewidth}
        \centering
        \includegraphics[width=\linewidth,  trim=50 195 50 190, clip]{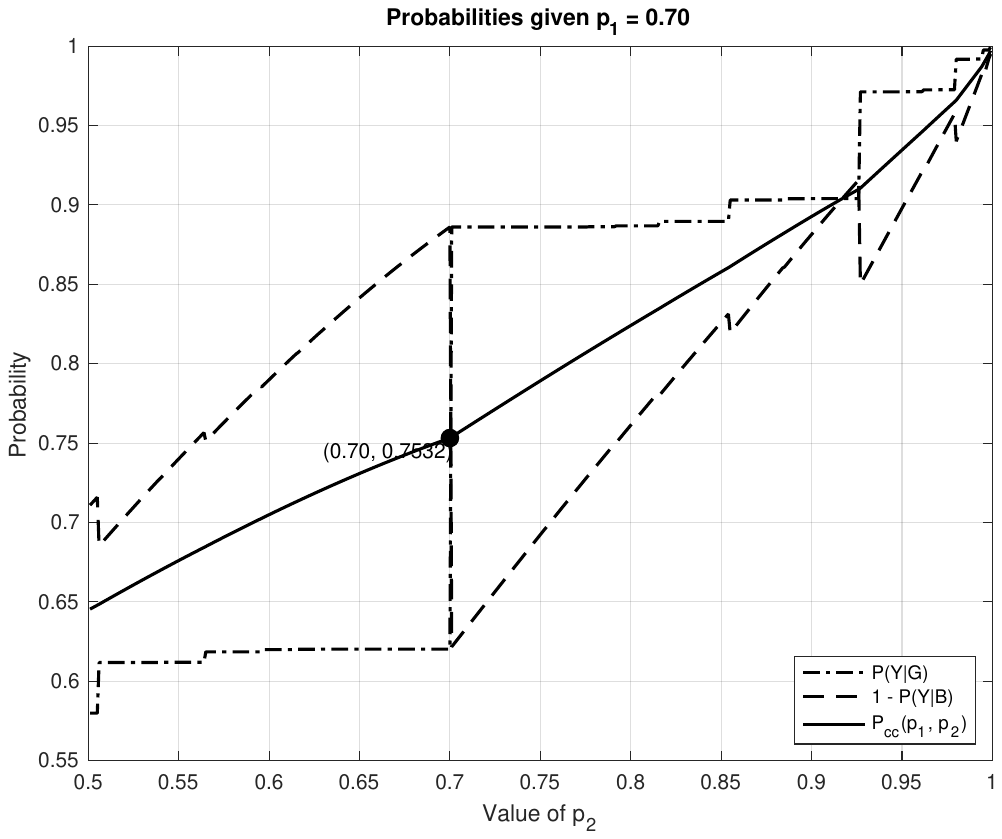}
        \setlength{\abovecaptionskip}{-5pt}
        \setlength{\belowcaptionskip}{0pt}
        \caption{Plot of $\prob_{\text{cc}}(p_1,p_2)$ computed using the formula for an irrational cascade constant, where $p_1 = 0.7$ and $p_2$ varies from $0.501$ to $0.999$ with step-size $0.001$.}
        \label{fig: a_values_increasing_p_2}
    \end{subfigure}
    \begin{subfigure}[b]{1\linewidth}
        \centering
        \includegraphics[width=\linewidth,  trim=50 195 50 190, clip]{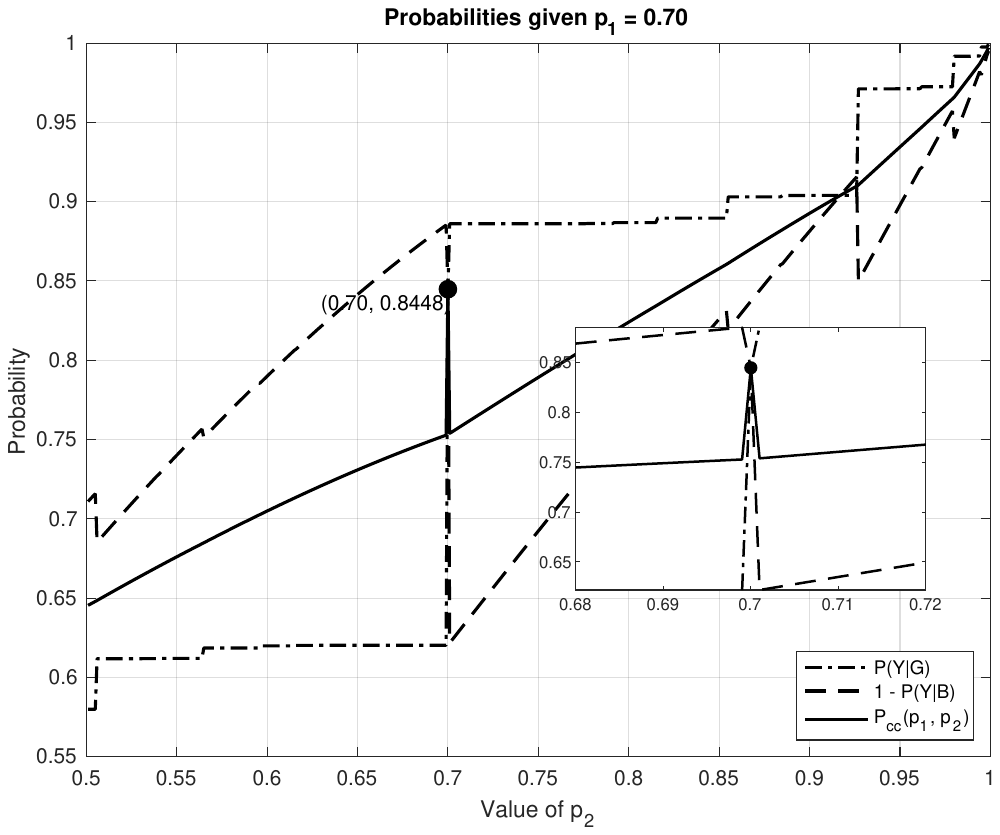}
        \setlength{\abovecaptionskip}{-5pt}
        \setlength{\belowcaptionskip}{0pt}
        \caption{Plot of $\prob_{\text{cc}}(p_1,p_2)$ computed using the formula for a rational cascade constant, where $p_1 = 0.7$ and $p_2$ varies from $0.501$ to $0.999$ with step-size $0.001$. An important difference compared to the irrational case is the behavior at $p_1=p_2=0.7$, where a noticeable discontinuity occurs.}
        \label{fig: a_values_increasing_p_1}
    \end{subfigure}
    \setlength{\abovecaptionskip}{-5pt}
    \setlength{\belowcaptionskip}{-10pt}
    \caption{Comparison of $\prob_{\text{cc}}(p_1,p_2)$ using the irrational and rational formulas with increasing $p_2$.}
    \label{fig:comparison}
\end{figure}

%% file: sections/07_Single_sideSignal_Quality_Improvement.tex
\section{Single-side Signal Quality Improvement}

Let us first consider a simpler problem than \eqref{eq:optimization}, where only a single signal quality can be improved. Specifically, either the \textit{Good} signal quality can be improved to $p_1' \in [p_1, p_1 + b]$, or the \textit{Bad} signal quality can be improved to $p_2' \in [p_2, p_2 + b]$, where $b$ represents the available budget. In this case, we demonstrate that when the budget is limited, the optimal strategy to maximize the probability of a correct cascade is either allocating the entire budget to the available option or to maintain symmetric signal qualities. While it may not be surprising that allocating the entire budget to a single option is optimal, maintaining symmetric signal qualities proves advantageous when the budget is small enough. This conclusion can be further generalized to the case where the budget is allocated across both signal qualities. In our analysis, we begin with the case where $p_2$ is the only quality that can be improved.


In the single-side signal quality improvement scenario, the cascade constant $a$ changes with the signal qualities $p_1$ and $p_2$. In particular, $a$ increases as $p_2$ improves. This property allows us to demonstrate that $\prob_{\text{cc}}(p_1, p_2)$ (under the irrational $a$ model) is a continuous non-decreasing function as shown in Figure~\ref{fig:comparison}.  Note that both $\prob(Y_{\mathrm{cas}} \mid G)$ and $\prob(N_{\mathrm{cas}} \mid B)$ exhibit a finite number of discontinuities as $p_2$ varies; however, it can be shown that when added to determine $\prob_{\text{cc}}(p_1, p_2)$, these ``cancel out'' leaving a continuous function.

\begin{prop} \label{prop: continuous, increasing P(cc), p_2}
    When the signal qualities are set as $p_1$ and $p_2+c$, $\prob_{\text{cc}}(p_1,p_2)$ under the irrational $a$ model is a continuous, non-decreasing function with respect to $c$ when $\frac{1}{2}< p_1< p_2+c < 1$.
\end{prop}

Next, we consider the behavior of $\prob_{\text{cc}}(p_1, p_2)$ when varying $p_1$.  

\begin{prop} \label{prop: continuous, increasing P(cc), p_1}
    When the signal qualities are set as $p_1 + c$ and $p_2$, $\prob_{\text{cc}}(p_1, p_2)$ under the irrational $a$ model is a continuous, non-decreasing function with respect to $c$ for $\frac{1}{2} < p_1 + c < p_2 < 1$. 
\end{prop}

At first glance, Proposition~\ref{prop: continuous, increasing P(cc), p_1} and Proposition~\ref{prop: continuous, increasing P(cc), p_2} suggest that if we have a budget to improve one signal quality, the optimal strategy is to allocate the entire budget to that improvement. However, a special case that has not been considered is the symmetric case, where $p_1 = p_2$. As shown in 
~\ref{fig:comparison}, at this rational value of $a$, there can be a large discontinuity in $\prob_{\text{cc}}(p_1, p_2)$ if it is calculated using the formulas in Proposition~\ref{prop: irrational prob}. If $p_1 = p_2$, then improving either $p_1$ or $p_2$ by a small enough amount would lead to a lower value of $\prob_{\text{cc}}(p_1, p_2).$  For example, in  Figure~\ref{fig:comparison}, when $p_1 = p_2 = 0.7$, the probability of a correct cascade is $0.844882$, and when $p_1 = 0.7$ and $0.7 < p_2 < 0.8$, the probability of a correct cascade is lower.

%% file: sections/08_Double_side_quality_improvement.tex
\section{Double-side Signal Quality Improvement}
Next, we consider the general problem where, given a budget, it is possible to improve both signal qualities simultaneously. How should the budget be allocated between the two signal qualities in this case?

Expressing the cascade constant and the probability of a correct cascade as functions of both $c_1$ and $c_2$ leads to a complex problem that is difficult to analyze. Therefore, we transform the problem into a simpler form. Suppose the budget is fully utilized, i.e., $p_1' + p_2' = p_1 + p_2 + b$, and that the budget is insufficient to improve $p_1$ after improving $p_2$ to 1, i.e., $b \leq 1 - p_2$. Under these conditions, we introduce a single variable $c$ with constraints $p_1' = \frac{p_1 + p_2 + b}{2} - c$, $p_2' = \frac{p_1 + p_2 + b}{2} + c$, where $0 < c \leq \frac{b + p_2 - p_1}{2}$, to represent the feasible region for $p_1'$ and $p_2'$. 

We demonstrate that, when the budget is fully utilized, the probability of a correct cascade is an increasing function of $c$.


\begin{prop} \label{prop: continuous, increasing P(cc), double-side}
    When the signal qualities are set as $p_1' = p-c$ and $p_2' = p+c$, $\prob_{\text{cc}}(p_1,p_2)$ is a continuous, non-decreasing function of $c$ when $\frac{1}{2}< p-c < p+c < 1$. 
\end{prop}

Having demonstrated the case where the budget $b$ must be exhausted, let us now examine the general case with constraint $b\leq1-p_2$. Without loss of generality, we assume that a total budget of $b'<b$ is utilized to improve the signal qualities in total, and $p_1' = \frac{p_1 + p_2 + b'}{2} - c$, $p_2' = \frac{p_1 + p_2 + b'}{2} + c$, where $0< c\leq \frac{b' + p_2 - p_1}{2}$. Then, Proposition~\ref{prop: continuous, increasing P(cc), double-side} suggests that it suffices to consider allocating the entire budget to improve $p_2$, i.e., let $p_2'=p_2+b'$. Notice that allocating the entire budget to improve $p_1$ is no better than allocating it to $p_2$: If $p_1+b'>p_2$, we set $c=\frac{p_1+b'-p_2}{2}$ to ensure $p_1'=p_1+b'$, else we set $c=-\frac{p_1+b'-p_2}{2}$, in both cases $c$ does not achieve the upper bound of it, $\frac{b' + p_2 - p_1}{2}$. Since $b' < b$, by Proposition~\ref{prop: continuous, increasing P(cc), p_2}, allocating the remaining budget $b-b'$ to $p_2'$ improves the probability of a correct cascade. Therefore, to optimize the probability of a correct cascade, the optimal approach is to allocate the entire budget to $p_2$, subject to the constraints $\frac{1}{2} < p_1 < p_2 < 1$ and $b\leq1-p_2$. Alternatively, under the original constraint $\frac{1}{2} < p_1, p_2 < 1$, another viable strategy is to set $p_1' = p_2' = \frac{p_1+p_2+b}{2}$ and focus on improving the overall signal qualities.

\begin{theorem}\label{thm: optimal strategy}
    Given a budget $b$, consider signal qualities $p_1' = p_1 + c_1$, and $p_2' = p_2 + c_2$, where $c_1,c_2 \geq 0$, $b < 1-p_2$ and $c_1 + c_2 \leq b$. Consider one strategy where $p_2' = p_2 + b$ and $p_1'=p_1$, and another strategy where $p_1'=p_2'=\frac{p_1+p_2+b}{2}$. At least one of of these strategies maximizes the probability of correct cascade.
\end{theorem}

The assumption we make about the budget that $b \leq 1 - p_2$ is introduced for analytical convenience. In the general case, if $1 - p_1 > b > 1 - p_2$, an optimal solution is one of the following: first improving $p_2$ to $1$ and then allocating the remaining budget to $p_1$; or distributing the budget equally to improve both signals to $\frac{p_1 + p_2 + b}{2}$. If $1 - p_1 \leq b$, an additional possibility arises: improving $p_1$ to $1$ and then allocating the remaining budget to $p_2$. All these scenarios can be derived by combining Proposition~\ref{prop: continuous, increasing P(cc), double-side}, Proposition~\ref{prop: continuous, increasing P(cc), p_1}, and Proposition~\ref{prop: continuous, increasing P(cc), p_2}.


%% file: sections/11_Appendix.tex
\section{Appendix}

\subsection{Proof of Lemma~\ref{lemma: action region}}
\begin{proof}
    If the agent takes \textit{No} action, corresponding to a right step of length $1$ in the random walk, then when the random walk reaches the interval $(a-1, a]$, it crosses the boundary, resulting in a cascade. Similarly, \textit{Yes} action when the random walk reaches the interval $[0, a-1)$ results in a cascade.
\end{proof}

\subsection{Proof of Lemma~\ref{lemma: irrational a never reach b_r-1}}
\begin{lemma}\label{lemma: irrational a never reach b_r-1}
    Suppose that the cascade constant $a$ is irrational, then for any observation history  that do not result in a cascade, the corresponding random walk never reaches the point $a-1$.
\end{lemma}

\begin{proof}
    Let's assume, for the sake of contradiction, that after $n_l$ left steps and $n_r$ right steps, the random walk reaches the point $a - 1 $. Then, the following equation must hold:
    \begin{align*}
        n_r-an_l &= a-1 \\
        n_r+1 &= a(n_l+1) \\
        a &= \frac{n_r+1}{n_l+1}.
    \end{align*}
    This leads to a contradiction to the assumption that $a$ is irrational.
\end{proof}

\subsection{Proof of Proposition~\ref{coro: prob of CC}}

\begin{proof}
     Notice that $\prob(Y_{\mathrm{cas}} \mid G) = \summ{i=0}{\infty}\prob^N_i(Y_{\mathrm{cas}} \mid G)$ and $\prob(Y_{\mathrm{cas}} \mid B) = \summ{i=0}{\infty}\prob^N_i(Y_{\mathrm{cas}} \mid B)$. Moreover, Lemma~\ref{lemma: irrational a never reach b_r-1} and Lemma~\ref{lemma: action region} imply that there is only one possible trajectory of the random walk which crosses the boundary and has a fixed number of \textit{No} actions. Consider an observation history leads to a Y-cascade. When the number of \textit{No} actions is $i$, the number of \textit{Yes} actions $n_Y$ must lead to a Y-cascade. The random walk model implies that 
     \begin{align*}
         i - an_Y &< -1 \\
         n_Y > \frac{i+1}{a}.
     \end{align*}
     Moreover, another constraint is that $n_Y-1$ \textit{Yes} actions does not lead to a Y-cascade. Therefore, it holds that,
     \begin{align*}
          i - a(n_Y-1) &\geq -1 \\
          n_Y \leq \frac{i+1}{a} + 1.
     \end{align*}
     Combining two constraints and the assumption that $n_Y$ is an integer, we obtain that $n_Y = \floor{\frac{i+1}{a}}+1$. When item is \textit{Good}, the probability of \textit{Yes} action is $p_1$, \textit{No} action is $1-p_1$. Therefore,
     $$\prob(Y_\mathrm{cas} \mid G) = \summ{i=0}{\infty}(1-p_1)^i p_1^{k_i}.$$
     Similarly,
     $$\prob(Y_{\mathrm{cas}} \mid B) = \sum_{i=0}^{p-1}{ p_2^i (1-p_2)^{k_i}}.$$
\end{proof}

\subsection{Proof of Proposition~\ref{prop: irrational prob}}

\begin{proof}
    We first calculate $\prob(Y_{\mathrm{cas}} \mid G)$, and $\prob(Y_{\mathrm{cas}} \mid B)$ is derived in a similar manner. Write the probability in a summation form: $\prob(Y_{\mathrm{cas}} \mid G) = \summ{i=0}{\infty}\prob^N_i(Y_{\mathrm{cas}} \mid G)$. Split the summation into two parts, $\prob(Y_{\mathrm{cas}} \mid G) = \summ{i=0}{r - 1}\prob^N_i(Y_{\mathrm{cas}} \mid G) + \summ{i=r}{\infty}\prob^N_i(Y_{\mathrm{cas}} \mid G)$. We then show how to derive the two parts separately. 
    
    Consider a walker in $\mathcal{M}(-1,a,a,1,p_1,1-p_1)$ starting from the origin point, suppose that the walker reaches the point $a-1$ with $n_l$ number of left(negative) steps and $n_r$ number of right(positive) steps, implying 
    $$n_l s_l + b_r - 1 = n_r s_r.$$
    Substitute the value of $s_l, s_r$ and $b_r$ into the equation, we obtain
    \begin{align*}
        n_la + a - 1 &= n_r \\
        (n_l + 1)a &= n_r + 1\\
        (n_l + 1)r &= q(n_r + 1).
    \end{align*}
    Because $n_l, n_r$ are non-negative integers and $p, q$ are co-prime, it holds that $n_l = q - 1$, $n_r = r - 1$ when the first time the walker reaches $b_r - 1$.

    Then it follows that $\prob^N_i(Y_{\mathrm{cas}} \mid G) = (1-p_2)^{k_i}p_2^{i}$ holds for $i \leq r - 1$ in the case, since before reaching $b_r - 1$, the walker cannot choose between left and right. This implies that 
    \begin{align}
        \summ{i=0}{r - 1}\prob^N_i(Y_{\mathrm{cas}} \mid G) = \summ{i=0}{r - 1}(1-p_1)^i p_1^{k_i}. \notag
    \end{align}

    Let us consider the remaining part. For simplicity we denote by $\prob^R_{\geq r}(C_l \mid G)$ the probability of the walker in $\mathcal{M}(-1,a,a,1,p_1,1-p_1)$ crossing the left boundary with no less than $r$ number of right steps. Then by definition $\summ{i=r}{\infty}\prob^N_i(Y_{\mathrm{cas}} \mid G) = \prob^R_{\geq r}(C_l \mid G)$. We define $\prob^R_{< r}(C_l \mid G)$ in similar way. Denote by $\mathcal{A}$ the event that the walker reaches the point $b_r - 1$ with $q - 1$ left steps and $r - 1$ right steps at time $t = r+q-2$, $\bar{\mathcal{A}}$ the complement event of $\mathcal{A}$, and $G$ the event that the item is \textit{Good}. By law of total probability, 
    \begin{align}
        \prob^R_{\geq r}(C_l \mid G) &=   \prob^R_{\geq r}(C_l\cap \mathcal{A} \mid G) + \prob^R_{\geq r}(C_l\cap \bar{\mathcal{A}} \mid G).\label{eq: decompose prob}
    \end{align}
    It holds that $\prob^R_{\geq r}(C_l\cap \bar{\mathcal{A}} \mid G) = 0$. Notice that if the walker cross the left boundary after no less than $r$ left steps, it must reach $b_r - 1$ after exactly $r-1$ left steps, because before reaching $b_r - 1$, the walker has no choice to remain within the boundaries according to Lemma~\ref{lemma: action region}. This also implies that $\prob(\mathcal{A} \mid G) = p_1^{q-1}(1-p_1)^{r-1}$.
    
    Then, rewrite $\prob^R_{\geq r}(C_l\cap \mathcal{A} \mid G)$ as $\prob^R_{\geq r}(C_l \mid G \cap \mathcal{A})\prob(\mathcal{A} \mid G)$. One key equality is that $\prob^R_{\geq r}(C_l \mid G \cap \mathcal{A}) = 2\prob(C_l \mid G)p_1(1-p_1)$. This is because that from the point $b_r - 1$, if the walker stays within the boundaries over the next two steps, there are only two possible choices: either move one step left followed by one step right, or the reverse. Both choices bring the walker back to the origin and add one step of rightward movement to the walker's history. Now $\prob^R_{\geq r}(C_l \mid G \cap \mathcal{A})$ simplifies to the probability of the walker crossing the left boundary starting from the origin. 

    From above discussions and (\ref{eq: decompose prob}) we obtain that 
    \begin{align*}
        \prob(C_l \mid G) &= \prob^R_{\geq r}(C_l \mid G) + \prob^R_{< r}(C_l \mid G)\\
        \prob(C_l \mid G)  &= 2\prob(C_l \mid G)\prob(\mathcal{A} \mid G)p_1(1-p_1)\\
        &\qquad \qquad + \prob^R_{< r}(C_l \mid G) \\
        \prob(C_l \mid G) &= \frac{\prob^R_{< r}(C_l \mid G)}{1 - 2p_1^{q}(1-p_1)^{r}} \\
        \prob(C_l \mid G) &= \frac{\summ{i=0}{r - 1}(1-p_1)^i p_1^{k_i}}{1 - 2p_1^{q}(1-p_1)^{r}}.
    \end{align*}
    This is equivalent to $\prob(Y_{\mathrm{cas}} \mid G) = \frac{\summ{i=0}{r - 1}(1-p_1)^i p_1^{k_i}}{1 - 2p_1^{q}(1-p_1)^{r}}$ by model reformulation.

    Follow the similar procedure $\prob(Y_{\mathrm{cas}} \mid B)$ can be obtained.
    
\end{proof}

\subsection{Proof of Proposition~\ref{prop: irrational case approaches rational case}}
\begin{proof}
    Notice that by assumption $a = \frac{r}{q}$, $p_1<1$ and $p_1>1/2$. When $r>k$, $0 < 2 p_1^q(1-p_1)^r < (1/2)^k$, therefore, 
    \begin{align}
         \summ{i=0}{r-1}(1-p_1)^i p_1^{k_i} < \frac{\summ{i=0}{r-1}(1-p_1)^i p_1^{k_i}}{1 - 2p_1^{q}(1-p_1)^{r}} \notag\\
         \quad< \left( \summ{i=0}{r-1}(1-p_1)^i p_1^{k_i}\right) \frac{2^k}{2^k-1}. \label{eq: ineq of rational case prob}
    \end{align}
    Also, since $\summ{i=r}{\infty}(1-p_1)^i p_1^{k_i} < \summ{i=r}{\infty}(1/2)^i < (1/2)^{k-1}$, we obtain that 
    \begin{align}
       \summ{i=0}{r-1}(1-p_1)^i p_1^{k_i} < \summ{i=0}{\infty}(1-p_1)^i p_1^{k_i} \notag\\
       \quad < \summ{i=0}{r-1}(1-p_1)^i p_1^{k_i} + (1/2)^{k-1}. \label{eq: ineq of irrational case prob}
    \end{align}
    From $\sum_{i=0}^{r-1}(1-p_1)^i p_1^{k_i} < \sum_{i=0}^{r-1}(1/2)^i < 1$, it holds that 
    \begin{align}
        \left( \summ{i=0}{r-1}(1-p_1)^i p_1^{k_i}\right) \left(\frac{2^k}{2^k-1} - 1\right) < \frac{1}{2^k-1}. \label{eq: ineq of diff}
    \end{align}
    (\ref{eq: ineq of rational case prob}), (\ref{eq: ineq of irrational case prob}), and (\ref{eq: ineq of diff}) implies that 
    \begin{align*}
        &\abs*{\frac{\sum_{i=0}^{r-1}(1-p_1)^i p_1^{k_i}}{1 - 2p_1^{q}(1-p_1)^{r}} - \sum_{i=0}^{\infty}(1-p_1)^i p_1^{k_i}} \\
        \quad &< \max\left\{\frac{1}{2^{k-1}}, \frac{1}{2^k - 1}\right\} = \frac{1}{2^{k-1}}.
    \end{align*}
    When $k > \log_2\left(\frac{1}{\epsilon}\right) + 1$, it is equivalent to 
    $$\abs*{\frac{\sum_{i=0}^{r-1}(1-p_1)^i p_1^{k_i}}{1 - 2p_1^{q}(1-p_1)^{r}} - \sum_{i=0}^{\infty}(1-p_1)^i p_1^{k_i}} < \epsilon.$$ 
    Notice that from the assumption that $p_1 < p_2$, it always holds that $a > 1$, which implies $r < p < k$. There are at most $\left(\log_2\left(\frac{1}{\epsilon}\right) + 1\right)^2$ rational numbers such that  
    \begin{align*}
        \abs*{\frac{\summ{i=0}{r-1}(1-p_1)^i p_1^{k_i}}{1 - 2p_1^{q}(1-p_1)^{r}} - \summ{i=0}{\infty}(1-p_1)^i p_1^{k_i}} > \epsilon.
    \end{align*}
    Similarly we can show there are also at most $\left(\log_2\left(\frac{1}{\epsilon}\right) + 1\right)^2$ rational numbers such that  
    \begin{align*}
        \abs*{\frac{\sum_{i=0}^{r-1}{ p_2^i (1-p_2)^{k_i}}}{1 - 2(1-p_2)^q p_2^{r}} -  \summ{i=0}{\infty}p_2^i (1-p_2)^{k_i}} > \epsilon.
    \end{align*}
\end{proof}

\subsection{Proof of Proposition~\ref{prop: continuous, increasing P(cc), p_2}}

\begin{lemma}
    $\summ{i=0}{\infty}(1-p_1)^i p_1^{k_i} - \summ{i=0}{\infty}(1-p_2)^{k_i} p_2^i$ is absolutely convergent when $\frac{1}{2}< p_1< p_2 < 1$. Therefore, 
    \begin{align}
        \prob_{\text{cc}}(p_1,p_2) &= \frac{\summ{i=0}{\infty}\left((1-p_1)^i(p_1)^{k_i} - (1-p_2)^{k_i}(p_2)^i\right)}{2}\notag \\
        &\qquad\qquad +\frac{1}{2}.
        \label{eq: exchange order of sum}
    \end{align}
\end{lemma}

\begin{proof}
    By assumption, $0<1-p_1<1/2$, $0<p_1<1$, implying that $\abs{(1-p_1)^i p_1^{k_i}} < 1/2^i$ for any $i\in \nat$. Also, $0<p_2<1$ implies that $\abs{(1-p_2)^{k_i}p_2^i} < p_2^i$ for any $i\in \nat$. The two series are absolutely convergent by definition, which implies the sum of two series is also absolutely convergent, and by exchanging the order of sum in Proposition~\ref{coro: prob of CC} we obtain (\ref{eq: exchange order of sum}).
\end{proof}

\begin{lemma}\label{lemma: a is increasing}
    Let $p_2' = p_2 + c$, $p_1<p_2'<1$, and $a = -\log(\frac{1-p_2'}{p_1})/\log(\frac{p_2'}{1-p_1})$, then $a$ is an increasing function of $c$ when fixing $p_1$ and $p_2$.
\end{lemma}

\begin{proof}
    Directly take the derivative with respect to $c$, we obtain that 
    \begin{align}
        a'(c) &= \frac{\log\left(\frac{-c - p_{2} + 1}{p_{1}}\right)}{\left(c + p_{2}\right) \log^{2}\left(\frac{1 - p_{1}}{c + p_{2}}\right)} - \frac{1}{\left(-c - p_{2} + 1\right) \log\left(\frac{1 - p_{1}}{c + p_{2}}\right)} \notag\\
        &= \frac{\log\left(\frac{-c - p_{2} + 1}{p_{1}}\right)(-c - p_{2} + 1) }{\log\left(\frac{1 - p_{1}}{c + p_{2}}\right)^2 (-c - p_{2} + 1) (c + p_{2})} \notag\\
        &\qquad\qquad -\frac{\left(c + p_{2}\right) \log\left(\frac{1 - p_{1}}{c + p_{2}}\right)}{\log\left(\frac{1 - p_{1}}{c + p_{2}}\right)^2 (-c - p_{2} + 1) (c + p_{2})}
        \label{eq: numerator of the derivative}
    \end{align}

    Calculate the derivative of the numerator, 
    \begin{align}
        &-\ln\left(\frac{-c - p_{2} + 1}{p_{1}}\right) - 1 - \ln\left(\frac{1 - p_{1}}{c + p_{2}}\right) + 1 \notag \\     
        &= -\ln\left(\frac{-c - p_{2} + 1}{p_{1}}\right) - \ln\left(\frac{1 - p_{1}}{c + p_{2}}\right) \label{eq: derivative_denominator}.
    \end{align}
    
    By assumption, $1-p_2-c < \frac{1}{2} < p_1$, $1-p_1 < \frac{1}{2} < c + p_2$, which implies that (\ref{eq: derivative_denominator}) is greater than $0$. Without loss of generality, we further assume that $p_2 = p_1$ and $c>0$, then when $c$ is small enough, it is not hard to deduce that (\ref{eq: numerator of the derivative}) is greater than $0$, therefore, (\ref{eq: numerator of the derivative}) is greater than $0$ when $c$ increases. This implies that $a$ is an increasing function with respect to $c$. For the case that $p_2 > p_1$, take $\Tilde{p_2 }= p_1$ and $\Tilde{c} = p_2-p_1+c$, and follow the similar steps we know that $a$ is still an increasing function with respect to $c$.
\end{proof}

\begin{lemma}\label{lemma: countinuous, increasing functions}
    For any $i\in\nat$, $(1-p_1)^i p_1^{k_i} - (1-p_2-c)^{k_i}(p_2+c)^i$ is a continuous, increasing function with respect to $c$ when $\frac{1}{2}< p_1< p_2+c < 1$.
\end{lemma}

\begin{proof}
    Consider an arbitrary $i$. We first show the continuity. Suppose the set $\mathcal{C}$ contains all the $c$ such that $k_i = \floor{\frac{i+1}{a}} + 1$ is an integer. When $\frac{1}{2}< p_1< p_2+c < 1$, $a$ is bounded and is an increasing function with respect to $c$ by Lemma~\ref{lemma: a is increasing}. Therefore, $\mathcal{C}$ is a set with finite elements, which in increasing order we assume to be $\{c_1, c_2, \ldots, c_k\}$. Then $(1-p_1)^i p_1^{k_i} + (1-p_2-c)^{k_i}(p_2+c)^i$ is continuous in the intervals $(c_j,c_j+1)$ for $j = 1,2,\ldots,k-1$.

    Let us consider the situation where $c$ take value at the points in $\mathcal{C}$. For simplicity we write $k_{i,n}$ as the value of $k_i$ and $a_n$ as the value of $a$ when $c = c_n$ for any $n\in \nat$. Notice that $k_{i,n+1} - k_{i,n} = 1$ for any $n\in\nat$. When $c=c_{n}$, 
    \begin{align}
        &(1-p_1)^i p_1^{k_i} - (1-p_2-c)^{k_i}(p_2+c)^i \notag\\
        = &(1-p_1)^i p_1^{k_{i,n}} - (1-p_2-c_n)^{k_{i,n}}(p_2+c_n)^i, \label{eq: value at c_n}
    \end{align}
    and when $c$ approaches $c_n$ from the left, the value of the function approaches 
    \begin{align}
        (1-p_1)^i p_1^{k_{i,n-1}} - (1-p_2-c_n)^{k_{i,n-1}}(p_2+c_n)^i.\label{eq: value at c_{n-1}}
    \end{align}
    The difference of (\ref{eq: value at c_n}) and (\ref{eq: value at c_{n-1}}) is 
    \begin{align}
        &- (p_2+c_n)^i \left((1-p_2-c_n)^{k_{i,n}}(1-p_2-c_n)^{k_{i,n-1}}\right) \notag \\
        &\qquad\qquad  +(1-p_1)^i\left(p_1^{k_{i,n}} - p_1^{k_{i,n-1}}\right)\notag \\
        &= -(1-p_1)^{i+1}p_1^{k_{i,n-1}} + (p_2+c_n)^{i+1}(1-p_2-c_n)^{k_{i,n-1}} \notag \\
        &= (1-p_1)^{i+1}p_1^{k_{i,n-1}} \notag \\
        &\qquad\qquad \cdot\left(\left(\frac{p_2+c_n}{1-p_1}\right)^{i+1} \left(\frac{1-p_2-c_n}{p_1}\right)^{k_{i,n-1}} - 1\right) \notag\\
        &= (1-p_1)^{i+1}p_1^{k_{i,n}-1} \notag \\
        &\qquad\qquad\cdot \left(\left(\frac{p_2+c_n}{1-p_1}\right)^{i+1} \left(\frac{1-p_2-c_n}{p_1}\right)^{k_{i,n}-1} - 1\right) \notag\\
        &= (1-p_1)^{i+1}p_1^{\frac{i+1}{a_n}}\notag\\
        &\qquad\qquad\cdot\left(\left(\frac{p_2+c_n}{1-p_1}\right)^{i+1} \left(\frac{1-p_2-c_n}{p_1}\right)^{\frac{i+1}{a_n}} - 1\right). \label{eq:diff_in_discontinuity}
    \end{align}

    Notice that by assumption, 
    $$a_n = -\log(\frac{1-p_2 - c_n}{p_1})/\log(\frac{p_2 + c_n}{1-p_1}),$$
    which implies that \eqref{eq:diff_in_discontinuity} equals $0$. Since  $c_n$ is chosen arbitrarily, the function is continuous at the points where $k_i$ is an integer. It is also known that the function is continuous in the intervals between the points. Therefore, we show the continuity.

    Next we will show the function is increasing. It is enough to show that the function is increasing in each interval $(c_j,c_j+1)$ for $j = 1,2,\ldots,k-1$ because the function is continuous. In each interval, $k_i$ is a constant, implying that $(1-p_1)^i p_1^{k_i}$ is also a constant. However, $(1-p_2-c)^{k_i}(p_2+c)^i$ is a decreasing function with respect to $c$ when fixing $k_i$. 
    
    To show this, we construct a coupling between two random walks $\mathcal{M}_1 = \mathcal{M}(-1,a,a,1,1-p_2,p_2)$ and $\mathcal{M}_2 = \mathcal{M}(-1,a,a,1,1-p_2-c,p_2+c)$. Suppose the walker in $\mathcal{M}_1$ is $W_1$, and the walker in $\mathcal{M}_2$ is $W_2$. The coupling ensures that $W_1$ and $W_2$ evolve together in a way that reflects their probabilities and maintains a direct relationship between the two walks. Use a random variable $U\sim\mathrm{Uniform}(0,1)$at each step to determine the direction for both $W_1$ and $W_2$. There are 3 cases: 1) If $U<p_2$, both $W_1$ and $W_2$ step right; 2) If $p_2<U<p_2+c$, $W_1$ step left and $W_2$ step right; 3) If $p_2+c<U$, both of them step left. In the coupling, whenever $W_2$ crosses the left boundary, $W_1$ must also have crossed the left boundary. Therefore, the equation $(1-p_2-c)^{k_i}(p_2+c)^i$ that evaluates the probability of crossing left boundary is a decreasing function.
    
    Combine the conclusion above, $(1-p_1)^i p_1^{k_i} - (1-p_2-c)^{k_i}(p_2+c)^i$ is an increasing function.
\end{proof}

\begin{proof}[Proof of Proposition~\ref{prop: continuous, increasing P(cc), p_2}]
    Denote that $h_i(c) = (1-p_1)^i p_1^{k_i} - (1-p_2-c)^{k_i}(p_2+c)^i$. Consider the case when $c$ belongs to an arbitrary closed interval which is a sub-interval of $(p_1 - p_2, 1-p_2)$. In such an closed interval, $a$ is always finite, which implies that $k_i\rightarrow +\infty$ when  $i\rightarrow +\infty$. Notice that 
    $$\abs{(1-p_1)^i p_1^{k_i} - (1-p_2-c)^{k_i}(p_2+c)^i} < \left(\frac{1}{2}\right)^i + \left(\frac{1}{2}\right)^{k_i}.$$
    For any $\epsilon>0$, there exists $N$, such that when $i_n,i_m > N$, $\abs{h_{i_n}-h_{i_m}} < \epsilon$. This, by definition, shows that $\summ{i=0}{\infty}h_i(c)$ is a uniformly convergent series. By Lemma~\ref{lemma: countinuous, increasing functions}, $\summ{i=0}{\infty}h_i(c)$ is a continuous, non-decreasing function, which is equivalent to that $\prob_{\text{cc}}(p_1,p_2)$ is a continuous, non-decreasing function concerning $c$. The closed interval is chosen arbitrarily; this further implies that $\prob_{\text{cc}}(p_1,p_2)$ is continuous and non-decreasing when $c\in(p_1 - p_2, 1-p_2)$. The domain is equivalent to that $\frac{1}{2}<p_1<p_2+c<1$.
\end{proof}

\subsection{Proof of Proposition~\ref{prop: continuous, increasing P(cc), p_1}}

\begin{lemma}\label{lemma: a is decreasing}
      Let $p_1' = p_1 + c$, $p_1<p_2'<1$, and $a = -\log(\frac{1-p_2'}{p_1})/\log(\frac{p_2'}{1-p_1})$, then $a$ is a decreasing function of $c$ when fixing $p_1$ and $p_2$.
\end{lemma}

\begin{proof}
    Take the gradient and we get
    \begin{align*}
        a'(c) = \frac{\log{\left(\frac{1-p_2}{p_1+c}\right)}(p_1+c) - \log{\left(\frac{1-p_1-c}{p_2}\right)}(1-p_1-c)}{\log^2{\left( \frac{1-p_1-c}{p_2} \right)(1-p_1-c)(p_1+c)}}.
    \end{align*}
    This is less than $0$ because $p_1+c<p_2$ by assumption. 
\end{proof}

\begin{proof}[Proof of Proposition~\ref{prop: continuous, increasing P(cc), p_1}]
    To prove the proposition, we follow the similar procedure as Proposition~\ref{prop: continuous, increasing P(cc), p_2}. By Lemma~\ref{lemma: a is decreasing}, there are at most countable values of $c$ such that $k_i$ is an integer. Follow the similar proof of Lemma~\ref{lemma: countinuous, increasing functions} we can show that $(1-p_1-c)^i (p_1+c)^{k_i} - (1-p_2)^{k_i}(p_2)^i$ is a continuous, increasing function with respect to $c$ when $\frac{1}{2}< p_1+c< p_2 < 1$. Finally, similar to the the proof of Proposition~\ref{prop: continuous, increasing P(cc), p_2} we can show the uniformly convergence and the non-decreasing property. 
\end{proof}

\subsection{Proof of Proposition~\ref{prop: continuous, increasing P(cc), double-side}}

\begin{lemma}\label{lemma: a is increaing, double-side}
     Suppose that $p_1' = p - c$, $p_2' = p + c$, then the function $a(c) = -\log(\frac{1-p_2'}{p_1'})/\log(\frac{p_2'}{1-p_1'})$ is an increasing function with respect to $c$ when $\frac{1}{2} < p_1' < p_2' < 1$.
\end{lemma}

\begin{proof}
    Combining Lemma~\ref{lemma: a is decreasing} and Lemma~\ref{lemma: a is increasing} implies the conclusion. 
\end{proof}

\begin{lemma}\label{lemma: increasing functions, double-side}
   Suppose that $p_1' = p - c$, $p_2' = p + c$, then for any $i\in\nat$ the function $(1-p_1')^i (p_1')^{k_i} - (1-p_2')^{k_i}(p_2')^i$ is a continuous function with respect to $c$ when $\frac{1}{2} < p_1' < p_2' < 1$.
\end{lemma}

\begin{proof}
    We follow the similar idea of the proof of Lemma~\ref{lemma: countinuous, increasing functions}. Since the cascade constant $a$ is a function of $c$, it follows that $k_i=\floor{\frac{i+1}{a}}+1$ is also a function of $c$. Consider an arbitrary $i\in\nat$. Let $\mathcal{C} = \{c_1,c_2,\ldots,c_k\}$ represent the set of all values of $c$ for which $k_i$ is an integer, where $c_{j}<c_{j+1}$ for $j=1,2,\ldots,k$. From Lemma~\ref{lemma: a is increaing, double-side} we know that the set is finite. For any $j = 1,2,\ldots,k$, the function is continuous in the interval $(c_{j}, c_{j+1})$ because $k_i$ is a constant in the interval. 
    
    Now consider an arbitrary element in $\mathcal{C}$. For simplicity we write $k_{i,n}$ as the value of $k_i$ and $a_n$ as the value of $a$ when $c = c_n$ for any $n\in \nat$. Notice that $k_{i,n+1} - k_{i,n} = 1$ for any $n\in\nat$. When $c=c_{n}$, 
    \begin{align}
        &(1- p + c)^i (p - c)^{k_i} - (1-p-c)^{k_i}(p+c)^i \notag\\
        = &(1-p+c_n)^i (p-c_n)^{k_{i,n}} - (1-p-c_n)^{k_{i,n}}(p+c_n)^i, \label{eq: value at c_n, double-side}
    \end{align}
    and when $c$ approaches $c_n$ from the left, the value of the function approaches 
    \begin{align}
        (1-p+c_n)^i (p-c_n)^{k_{i,n-1}} - (1-p-c_n)^{k_{i,n-1}}(p+c_n)^i.\label{eq: value at c_{n-1}, double-side}
    \end{align}
    The difference of (\ref{eq: value at c_n, double-side}) and (\ref{eq: value at c_{n-1}, double-side}) is 
    \begin{align}
        & -(1-p+c_n)^{i+1}(p-c_n)^{k_{i,n-1}} \notag\\
        &\qquad\qquad+ (p+c_n)^{i+1}(1-p-c_n)^{k_{i,n-1}} \notag \\
        &= (1-p+c_n)^{i+1}(p-c_n)^{k_{i,n-1}} \notag \\
        &\qquad\qquad\cdot\left(\left(\frac{p+c_n}{1-p+c_n}\right)^{i+1} \left(\frac{1-p-c_n}{p-c_n}\right)^{k_{i,n-1}} - 1\right) \notag \\
        &= (1-p+c_n)^{i+1}p^{k_{i,n}-1} \notag \\
        &\qquad\qquad\cdot \left(\left(\frac{p+c_n}{1-p+c_n}\right)^{i+1} \left(\frac{1-p-c_n}{p-c_n}\right)^{k_{i,n}-1} - 1\right) \notag \\
        &= (1-p+c_n)^{i+1}p^{\frac{i+1}{a_n}} \notag \\
        &\qquad\qquad\cdot\left(\left(\frac{p+c_n}{1-p+c_n}\right)^{i+1} \left(\frac{1-p-c_n}{p-c_n}\right)^{\frac{i+1}{a_n}} - 1\right).\label{eq:diff_in_discontinuity, double-side}
    \end{align}

    Notice that by assumption, 
    $$a_n = -\log(\frac{1-p - c_n}{p-c_n})/\log(\frac{p + c_n}{1-p+c_n}),$$
    which implies that \eqref{eq:diff_in_discontinuity, double-side} equals $0$. Since  $c_n$ is chosen arbitrarily, the function is continuous at the points where $k_i$ is an integer. It is also known that the function is continuous in the intervals between the points. Therefore, we show the continuity.
\end{proof}

\begin{lemma}
    Suppose that $p_1' = p - c$, $p' = p + c$, then for any $i\in\nat$ the function $(1-p_1')^i (p_1')^{k_i} - (1-p_2')^{k_i}(p_2')^i$ is an increasing function of $c$ when $\frac{1}{2} < p_1' < p_2' < 1$.
\end{lemma}

\begin{proof}
    Denote $f(c) = (1-p_1')^i (p_1')^{k_i}$, $g(c) = (1-p_2')^{k_i}(p_2')^i$. Consider an arbitrary $i\in\nat$. Similar to Lemma~\ref{lemma: increasing functions, double-side}, let $\mathcal{C} = \{c_1,c_2,\ldots,c_k\}$ represent the set of all values of $c$ for which $k_i$ is an integer, where $c_{j}<c_{j+1}$ for $j=1,2,\ldots,k$. The function is continuous in the interval $(c_{j}, c_{j+1})$ for $j=1,2,\ldots,k$. Consider an arbitrary interval $(c_{n}, c_{n+1})$, and the corresponding value of $a$ in the interval is $a_n = -\log(\frac{1-p-c_n}{p-c_n})/\log(\frac{p+c_n}{1-p+c_n})$. It follows from Lemma~\ref{lemma: a is increaing, double-side} and $c\in(c_n,c_{n+1})$ that 
    $$a_n < -\log(\frac{1-p-c}{p-c})/\log(\frac{p+c}{1-p+c}) = a(c).$$
    Take the derivative of $f(c)$ and $g(c)$, it follows that  
    \begin{align}
        f'(c) &= i(1-p+c)^{i-1}(p-c)^{k_i} \notag \\
        &\qquad\qquad - k_i(1-p+c)^i(p-c)^{k_i-1} \notag\\
        &= (1-p+c)^{i-1}(p-c)^{k_i-1}\notag \\
        &\qquad\qquad\cdot(i(p-c)-k_i(1-p+c)),\label{eq: f'(c)}
    \end{align}
    and 
    \begin{align}
        g'(c) &= i(p+c)^{i-1}(1-p-c)^{k_i} \notag \\
        &\qquad\qquad- k_i(1-p-c)^{k_i-1}(p+c)^i \notag\\
        & = (p+c)^{i-1}(1-p-c)^{k_i-1}\notag \\
        &\qquad\qquad\cdot(i(1-p-c) - k_i(p+c)).\label{eq: g'(c)}
    \end{align}
    The quotient of \eqref{eq: f'(c)} and \eqref{eq: g'(c)} is 
    \begin{align}
        &\left(\frac{1-p+c}{p+c}\right)^{i-1}\left(\frac{p-c}{1-p-c}\right)^{k_i-1} \notag \\
        &\qquad\qquad\cdot\left(\frac{i(p-c)-k_i(1-p+c)}{i(1-p-c) - k_i(p+c)}\right) \notag \\
        &= \left(\frac{1-p+c}{p+c}\right)^{i-1}\left(\frac{p-c}{1-p-c}\right)^{\frac{i+1}{a_n}}\notag \\
        &\qquad\qquad\cdot \left(\frac{i(p-c)-k_i(1-p+c)}{i(1-p-c) - k_i(p+c)}\right) \notag \\
        &> \left(\frac{1-p+c}{p+c}\right)^{i-1}\left(\frac{p-c}{1-p-c}\right)^{\frac{i}{a}}\notag \\
        &\qquad\qquad\cdot\left(\frac{i(p-c)-k_i(1-p+c)}{i(1-p-c) - k_i(p+c)}\right) \notag \\
        &= \left(\frac{1-p+c}{p+c}\right)^{i-1} \left(\frac{1-p+c}{p+c}\right)^{-i}\notag \\
        &\qquad\qquad\cdot\left(\frac{i(p-c)-k_i(1-p+c)}{i(1-p-c) - k_i(p+c)}\right) \notag \\
        &= \frac{p+c}{1-p+c}\left(\frac{i(p-c)-k_i(1-p+c)}{i(1-p-c) - k_i(p+c)}\right) \notag \\
        &= \frac{i\frac{(p+c)(p-c)}{1-p+c}-k_i(p+c)}{i(1-p-c) - k_i(p+c)}.\notag
    \end{align}

    Notice that $\frac{1}{2}<p-c<p+c<1$, which implies that $\frac{(p+c)(p-c)}{1-p+c} > (1-p-c)$. Therefore, suppose that 
    $$g'(c) = \left(i(1-p-c) - k_i(p+c)\right)z(c),$$ 
    for some function $z(c)$, then 
    $$f'(c) > \left(i\frac{(p+c)(p-c)}{1-p+c}-k_i(p+c)\right)z(c).$$
    It then follows that $f'(c)-g'(c)>0$, indicating that  $(1-p_1')^i (p_1')^{k_i} - (1-p_2')^{k_i}(p_2')^i$ increases as $c$ increases.
\end{proof}

\begin{proof}[Proof of Proposition~\ref{prop: continuous, increasing P(cc), double-side}]
    Similar to proposition~\ref{prop: continuous, increasing P(cc), p_2}, we can show the uniformly convergence and the non-increasing property.
\end{proof}

\subsection{Proof of Theorem~\ref{thm: optimal strategy}}

\begin{proof}
    The only case not discussed is when $p_1'=p_2'$. However, by Proposition~\ref{prop: continuous, increasing P(cc), p_2} and Proposition~\ref{prop: continuous, increasing P(cc), p_2}, the probability of correct cascade increases with the improvement of single-side signal qualities. This indicates that the $p_1'=p_2'=\frac{p_1+p_2+b}{2}$ is an optimal strategy if having $p_1'=p_2'=p$ for some $p$ is optimal. 
\end{proof}